\documentclass{article}
\usepackage{ijcai22}

\usepackage{times}

\usepackage{soul}
\usepackage{url}
\usepackage[hidelinks]{hyperref}
\usepackage[utf8]{inputenc}
\usepackage[small]{caption}
\usepackage{graphicx}
\usepackage{amsmath}
\usepackage{booktabs}
\usepackage{amssymb} 
\usepackage{comment}
\usepackage{subcaption}
\usepackage{multirow,enumitem}
\usepackage{algorithm}
\usepackage[noend]{algorithmic}
\usepackage{amsthm}
\usepackage{xspace}
\usepackage{xcolor}
\newtheorem{theorem}{Theorem}
\newtheorem{definition}{Definition}
\newtheorem{prob}{Problem}
\def\wgr{g}
\def\wg{G}
\def\wt{W}
\def\at{A}
\def\dw{w}
\def\ew{e}
\def\pw{p}
\def\pmin{P_{min}}
\def\war{\omega}
\def\cp{C}
\def\hw{H}
\def\wp{P}
\def\wgalgo{\textsc{Work4Food}\xspace}
\def\fmalgo{\textsc{FoodMatch}\xspace}
\def\ffalgo{\textsc{FairFoody}\xspace}
\setlist{nolistsep,leftmargin=*}

\newcounter{casenum}
\newenvironment{caseof}{\setcounter{casenum}{1}}{\vskip.5\baselineskip}
\newcommand{\case}[2]{\vskip.5\baselineskip\par\noindent {\bfseries Case \arabic{casenum}:} #1.~#2\addtocounter{casenum}{1}}

\title{Gigs with Guarantees: Achieving Fair Wage for Food Delivery Workers}

\author{
Ashish Nair\and Rahul Yadav\and Anjali Gupta\and Abhijnan Chakraborty\and Sayan Ranu\And Amitabha Bagchi\\
\affiliations
Indian Institute of Technology Delhi\\
\emails
\{Ashish.R.Nair.cs517, Rahul.Yadav.cs517, anjali, abhijnan, sayanranu, bagchi\}@cse.iitd.ac.in
}
\begin{document}

\maketitle

\begin{abstract}
With the increasing popularity of food delivery platforms, it has become pertinent to look into the working conditions of the `gig' workers in these platforms, especially providing them fair wages, reasonable working hours, and transparency on work availability. However, any solution to these problems must not degrade customer experience and be cost-effective to ensure that platforms are willing to adopt them. We propose \wgalgo, which provides income guarantees to delivery agents, while minimizing platform costs and ensuring customer satisfaction. \wgalgo  ensures that the income guarantees are met in such a way that it does not lead to increased working hours or degrade environmental impact. To incorporate these objectives, \wgalgo  balances supply and demand by controlling the number of agents in the system and providing dynamic payment guarantees to agents based on factors such as agent location,  ratings, etc. We evaluate \wgalgo on a real-world dataset from a leading food delivery platform and establish its  advantages over the state of the art in terms of the multi-dimensional objectives at hand.
\end{abstract}

% \vspace{-0.10in}
\section{Introduction and Related Work}
Food delivery platforms like Swiggy, Zomato, GrubHub or Deliveroo have become an extremely popular choice among customers to order and get food delivered to them. Alongside offering increased business to the restaurants, they also provide a livelihood to thousands of delivery agents, who pick the ordered food from the restaurants and deliver them at the customers' doorsteps. In developing countries with high unemployment rates, despite the `gig' nature of delivery jobs, these platforms have become the only source of income for the majority of delivery agents~\cite{fairwork,foodDeliverySA}. However, a range of issues are presently plaguing the food delivery industry -- poor working conditions of the agents, pressure of on-time delivery while navigating heavy traffic, opaque job assignments, etc.~\cite{zhou2020digital}. Specifically, a major concern of the delivery agents is their inadequate income against the backdrop of rising cost of fuel and maintenance\footnote{A delivery agent typically gets a small delivery fee per order,  except occasional tips and incentives~\cite{foodDeliveryNewsMinute}.}, forcing many agents to go on repeated strikes to demand better pay~\cite{swiggyStrike,turkeyStrike}. In fact, a non-profit labor watchdog organization Fairwork ({\tt fair.work}) found that none of the food delivery platforms in India ensures legal minimum wage, even if an agent works for more than 10 hours a day %for the platform~
\cite{fairwork}.

\begin{figure}[t]
\centering
% \vspace{-0.20in}
\includegraphics[width=0.75\linewidth]{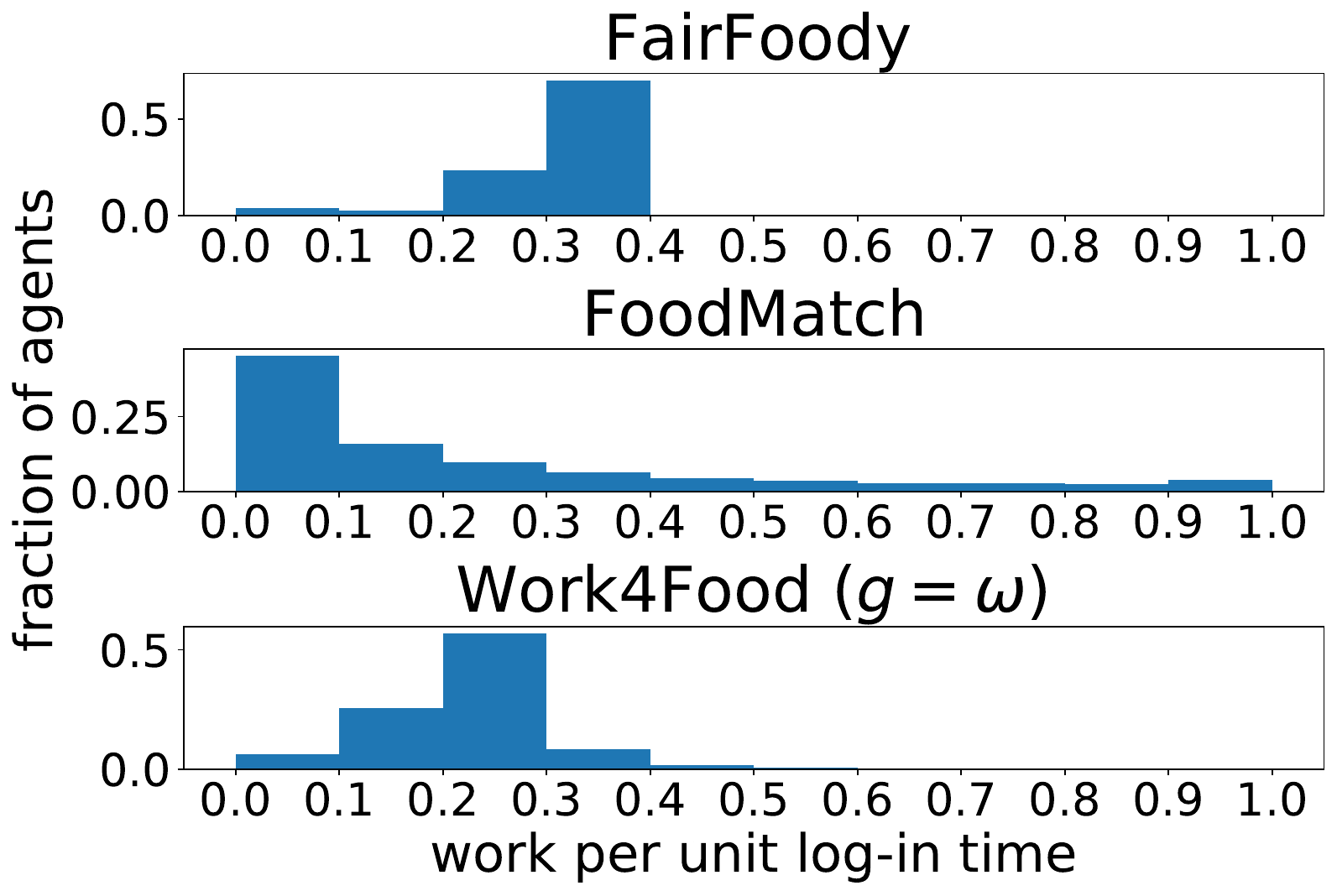}
% \vspace{-0.10in}
\caption{Work distribution in \ffalgo, \fmalgo and our proposal \wgalgo (at a specific configuration $\wgr=\war$ described in Section \ref{sec:work_guarantee}) in city B of the dataset. An agent is said to be working if it is servicing an order.} %only 1 day
\label{fig:work_distribution}
% \vspace{-0.20in}
\end{figure}

%\textcolor{blue}{International labor organization reported that deficient work is a prevailing problem for web-based crowdsourcing workers, leading to their income instability\cite{LabourOrg}}.
%The food delivery problem involves four of parties: customers, restaurants, delivery agents and a food delivery platform which assigns orders to delivery agents. Each of these parties have different objectives which are hard to optimize simultaneously. Customers want that their orders be delivered as quickly as possible. Delivery agents want that they be allotted enough and fair work. The platform would to minimize it's costs while ensuring customers are satisfied. 

Earlier works in this space have mostly attempted to minimize the order delivery time alone~\cite{Kottakki2020CustomerED,Ulmer2021TheRM}, by pre-placing delivery agents in different parts of the city anticipating demand~\cite{Xue2021OptimizationOR}, or by minimizing delivery agents' waiting time~\cite{Weng2021LaborrightPD}. \fmalgo~\cite{foodmatch_full,foodmatch} overcame many simplifying assumptions made in prior works (e.g., perfect order arrival information~\cite{article:mealDelivery}, neglecting the road network~\cite{mdrp} and food preparation time~\cite{10.14778/3368289.3368297}), and proposed a realistic and scalable solution. While \fmalgo minimized the delivery time, ~\cite{fairfoody} showed that such one-sided optimization leads to {\it unfair} work distribution among delivery agents, resulting in unequal incomes. They further proposed a multi-objective algorithm \ffalgo to provide {\it fair distribution of income opportunities} among the agents, while ensuring minimal increase in delivery times. To our knowledge, \ffalgo is the only prior work looking into the fairness of delivery agents' incomes.

%The earlier works in this space, \fmalgo  and   tried to optimize on delivery times and income fairness respectively. \fmalgo had very low delivery times but the  wasn't fair which in turn leads to  \ffalgo drastically improved on the income fairness score without impacting the delivery times much. 

\ffalgo tries to ensure that all agents get the same amount of work for every unit of time they are working for the platform. We plot the work distribution in \fmalgo and \ffalgo in Fig.~\ref{fig:work_distribution} when run on an identical stream of orders. In \ffalgo, all agents get concentrated around the same work per unit log-in time (0.3 -- 0.4). In sharp contrast, \fmalgo creates an allocation where some agents get most of the work, and a large number get very little work.

``Work'' on a food delivery platform consists of three components: (1) driving to the restaurant (\textit{first-mile}), (2) \textit{waiting} at the restaurant while food is being prepared, and (3) driving from restaurant to the customer's location (\textit{last-mile}). Note that the first-mile time and food preparation time can proceed in parallel. While \fmalgo pre-dominantly assigns an order to the closest delivery agent from the restaurant's location, \ffalgo assigns work to a relatively far-away driver so that the driver reaches the restaurant within the waiting time, and therefore, minimally affect the final delivery time. Through this design, the location of an agent plays a less central role in determining the chances of being allocated an order, and consequently provides \ffalgo more flexibility in attaining equitable work distribution.
  %\wgalgo also concentrates the work distribution (slightly to the left of \ffalgo) but does not increase the total work done by the agents.  
  
While the distribution of work certainly looks fairer in \ffalgo, there is a downside to this. As shown in Table~\ref{tab:fm_ff_intro}, \ffalgo makes agents work significantly more to deliver the same set of orders to ensure fairness. Thus, higher work to deliver the same set of orders increases the platform's cost in terms of delivery payments, fuel costs for agents, and has a negative environmental impact.
  
\begin{table}[t]
\centering
% \vspace{-0.10in}
\scalebox{0.7}{
\begin{tabular}{lrrr}
\toprule
\multirow{2}{*}{Property}  &\multirow{2}{*}{\ffalgo} & \multirow{2}{*}{\fmalgo} & \wgalgo\\
& & & $\wgr = \war$\\
\midrule
Avg. Delivery Time &  16.08 mins  & 15.91 mins &  16.68 mins \\
Gini Income/log-in Time & 0.09  & 0.63 &  0.11 \\
Avg. Work per Agent & 3.26 hrs & 2.31 hrs & 2.19 hrs\\
Cost (in 10E6 $\pw$ units) & 2.46 & 1.75 & 1.85 = 1.64 + 0.21\\
\bottomrule
\end{tabular}}
% \vspace{-0.10in}
\caption{Performance of \ffalgo, \fmalgo, and our proposal \wgalgo in city B of the dataset. An order's delivery time is the time between the order being placed and delivered. We measure inequality among agent incomes (per unit of log-in time) with Gini scores. The lower the Gini score, the better the fairness. $\pw$ is the pay of an agent per hour worked. Cost for \wgalgo includes handouts = 0.21 10E6 (Def. \ref{def:cp}). For \ffalgo and \fmalgo, cost here is the total payment for work.} 
\label{tab:fm_ff_intro}
% \vspace{-0.20in}
\end{table}
 
%The above analysis shows that equitable distribution of work comes at a cost. We therefore 
%To overcome this cost, 
In this context, we ask an important question: \textit{Can equitable distribution be achieved without the pitfalls of additional greenhouse emissions, higher cost to the delivery platform, and a more relaxed work distribution among agents?} We design an algorithm called \wgalgo, which shows that the above objective can indeed be achieved (See Fig.~\ref{fig:work_distribution} and Table~\ref{tab:fm_ff_intro}). Our key contributions are as follows:
\begin{itemize}
     \item {\bf Novel problem: } We propose a more realistic approach to provide fair incomes by ensuring that all agents get income guarantees based on government-mandated minimum wage. In addition, the proposed formulation provides sufficient control to the platform to decide the %sufficient
     number of agents to on-board so that the income guarantee can be met, which includes the provision of agent-personalized guarantees (\S~\ref{sec:formulation}).
     \item {\bf Algorithm design: } %\textbf{[Ashish/Rahul] Please write a brief 2-3 sentence summary. Take a look \fmalgo for an example.} 
     We design an algorithm called \wgalgo, which allows us to set a payment guarantee that is \textit{provably optimal} for minimizing cost for the platform. Powered by a novel combination of \textit{minimum weight bipartite matching} with \textit{Gaussian process regression}, \wgalgo analyzes the demand-supply dynamics in the system and generates an allocation that balances the \textit{triple} needs of minimizing delivery time, minimizing cost for the platform, and ensuring payment guarantees %are met 
     for the delivery agents (\S~\ref{sec:algo}).
     % are  and  first provide each agent with the same or different work guarantees based on minimum wage requirements, regression model predictions, or agent ratings when they on-board. We predict agent work based on parameters like order density and agent availability using a regression model which is trained offline and reject any agents that cannot meet their guarantees. We aggregate orders in a window and create a bipartite matching graph with active vehicles to perform assignments. The edge weight of the graph prioritizes underpaid agents and simultaneously minimizes platform costs.
     \item {\bf Evaluation: } We evaluate \wgalgo on a real food delivery dataset from %one of the largest food delivery aggregators in the world. 
     a large food delivery platform. Our experiments reveal that \wgalgo ensures the minimum wage guarantee with high probability, improves cost margin of the platform and achieves fairness without inflating work per unit time. More importantly, \wgalgo empowers policymakers with a tool that can be deployed to regulate this business (\S~\ref{sec:experiments}).
\end{itemize}

%Ensuring that all agents earn the same income may not always be the best idea as there may be more (/less) work available at certain times of the day and in specific locations. Also, the platform may wish to have 

%Our proposed algorithm \wgalgo can guarantee work to agents based on several parameters like order density, agent availability, and agent ratings while minimizing platform costs, providing transparency to agents, and maintaining reasonable delivery times. We also explore rejecting agents from on-boarding the system when sufficient work is not available to avoid agents from being underpaid. 

%\textcolor{blue}{[Need a final paragraph summarizing experimental results.]}
% \vspace{-0.10in}
\section{Preliminaries and Problem Formulation}
\label{sec:formulation}
% orders, agents, batching, road network, delivery time, order windows, 
The food delivery problem involves four parties: customers, restaurants, delivery agents, and the delivery platform (a.k.a. the `system'). The platform receives a stream of orders, which it then assigns to the agents. While doing so, we keep a few constraints in mind – agents have a fixed capacity to carry orders, and the platform does not assign orders that cannot get delivered within the service level agreement (SLA). The SLA is the maximum delivery time that the platform promises customers.
Each agent logs into the platform and informs how long they plan to stay. %The %delivery system works in different cities with platform has access to the road network data like travel times and restaurant locations in different cities.  
% \vspace{-0.05in}
\subsection{Dataset} 
\label{sec:dataset}
We use a licensed food delivery dataset released by~\cite{foodmatch}\footnote{An anonymized version of the proprietary dataset can be obtained from \url{https://www.cse.iitd.ac.in/~sayan/files/foodmatch.txt}. 
 % {\color{black} \url{ https://bit.ly/3HGdavB}}
 %https://drive.google.com/drive/folders/17e6ZXzD8e9lAdnBEL8aT09MLJDrv1Yh8?usp=sharing" or "https://bit.ly/3HGdavB".
 The name of the cities have not been revealed by the service provider.}. %It consists of six days of data provided by one of India's leading meal delivery service companies from three large metropolitan cities.
The dataset is sourced from a leading food delivery service in India and consists of eighteen days worth of delivery data from three large metropolitan cities. Table~\ref{tab:dataset} provides a summary of the same. The dataset comprises of: 
\begin{itemize}
\item The \textbf{Road network} of the three cities, obtained from {\tt OpenStreetMap.org} along with the average speed in each road segment at different times of the day.
\item \textbf{GPS pings} of delivery agents, which are \emph{map-matched} to the road network to obtain network-aligned trajectories~\cite{mapmatch}.
\item \textbf{Order information} such as customer and restaurant locations and %along with the 
average food preparation time in restaurants.
%\item Delivery vehicle trajectories.
%\item Data describing multiple factors such as vehicle IDs, average speed in each road segment at different hours, information on each received order like restaurant and customer locations, mean food preparation time in each restaurant, and so on. 
\end{itemize}
%We match the vehicle GPS pings to the road network to obtain network-aligned trajectories~\cite{mapmatch}.

\begin{table}[t]
\centering
% \vspace{-0.10in}
\scalebox{0.8} {
\begin{tabular}{lrrrrrr} 
    \toprule
    City & \# Rest- & \# Vehicles & \# Orders &  \# & \# &  Popu-\\ 
    & aurants & (avg./day) & (avg./day) & Nodes & Edges & lation\\
    \midrule
    A & $2085$ & $2454$ & $23442$ & 39k & 97k & $>$5M\\
    B  & $6777$ & $13429$ & $159160$ & 116k & 299k & $>$8M\\ 
    C & $8116$ & $10608$ & $112745$ & 183k & 460k & $>$8M\\  
 \bottomrule
\end{tabular}}
% \vspace{-0.05in}
\caption{Overview of the dataset.}
\label{tab:dataset}
% \vspace{-0.20in}
\end{table}
% \vspace{-0.10in}
\subsection{Formulation}
We now define some terms useful for our proposal. Table \ref{tab:notations} has a summary of the frequently used notations in the paper.

%Any allocation algorithm for solving the food delivery problem assigns incoming orders to available active agents. 
\begin{definition}[Delivery Times – EDT, SDT, XDT]
The {\em expected delivery time} $EDT(o,v)$ of an order $o$ is the expected time it would take to be delivered if it is assigned to agent $v$. 

The {\em shortest delivery time} $SDT(o)$ of order $o$ is the fastest it can be delivered if some vehicle could serve it with no wait time or other detour delays. It is the sum of the food preparation time and the shortest travel time between the restaurant and customer's location. 

The {\em excess delivery time} $XDT(o,v)$ of an order $o$ with respect to an agent $v$ is the difference between $EDT(o,v)$ and $SDT(o)$. We compute these delivery times using standard graph algorithms as in ~\cite{foodmatch}.
\label{def:edt_sdt_xdt}
\end{definition} 

\begin{prob}[Minimizing delivery time]
\label{prb:time}
Given a set of agents $V^t$ and unallocated orders $O^t$ at timestamp $t$, find an allocation of  orders to agents so that  $\sum_{\forall o\in O^t}XDT(o,A(o))$  is minimized. $A(o)$ denotes the agent allocated to order $o$.
\end{prob}
\fmalgo~\cite{foodmatch} studies the above problem and proposes an effective solution. 

\begin{definition}[Work and Active Times]
\label{def:work}
{\em Work time} $\wt_v^t$ is the time spent by a delivery agent $v$ till time $t$ servicing orders, i.e., either traveling or waiting for assigned orders. The {\em active time} $\at_v^t$ (in hours) of the agent is the time for which they log in to the system till time $t$. $\wt_v$ and $\at_v$ are the total {\em work time} and total {\em active time} of agent $v$ over their entire shift.
\end{definition}

In all of the models we compare, the agent incomes are proportional to their work times. So in order to guarantee income to agents, our new proposed algorithm guarantees that they get enough work. 

\begin{definition}[Work Guarantee Ratio]
\label{def:wgr}
The work guarantee ratio $\wgr_v$ of a delivery agent $v$ is the amount of work time guaranteed by the system to agent $v$ per unit active time.
% The {\em work guarantee ratio} $\wgr_v$ for a delivery agent $v$ is the fraction of time the platform guarantees that the agent will be working in the system for every unit of active time. 
Note a work guarantee naturally translates into an income guarantee since the payment is a function of the work.
\end{definition} 

\begin{definition}[Work Guarantee]
\label{def:wg}
The {\em work guarantee} $\wg_v^t$ for a delivery agent $v$ is the amount of work (in hours) the platform guarantees to the agent till time $t$, i.e., $\wg_v^t = \wgr_v \times \at_v^t$. The total {\em work guarantee}  $\wg_v$ of agent $v$ is given by $\wg_v = \wgr_v \times \at_v$.
\end{definition}

\begin{definition}[Work Payment, Handout and Platform Cost]
\label{def:cp}

All agents are paid the same rate $\pw$ for every hour worked in our models. So, the work payment $\wp_v$ for any agent $v$ is given by $\wp_v  =  \pw \times \wt_v$.

We envision a system where the platform hands out money to agents in lieu of any unmet work guarantees. So, the handout $\hw_v$ for any agent $v$ is given by:
% \vspace{-0.05in}
\begin{equation}
\label{eq:handout}
\hw_v  =  \pw \times \max(0,\wg_v - \wt_v)
\end{equation}

The {\em platform cost} $\cp$ is the total money the platform has to pay the agents for their work and any unsatisfied work guarantee. It is given by:
% \vspace{-0.10in}
\begin{equation}
\label{eq:cp}
\cp  =  \sum_v \pw \wt_v + \sum_v \hw_v
\end{equation}
\end{definition} 
\begin{table}[t]
\centering
% \vspace{-0.10in}
\scalebox{0.65}{
\begin{tabular}{cl}
\toprule
\textbf{Notation} & \textbf{Description}\\
\midrule
$\wt_v^{t}, \wt_v$ & Work time of agent $v$ till time $t$, total work time of agent $v$\\
$\at_v^{t}, \at_v$ & Active time of agent $v$ till time $t$, total active time of agent $v$\\
$\wgr, \wgr_v$ & Fixed work guarantee ratio for all agents, work guarantee ratio for agent $v$\\
$\wg_v^{t}, \wg_v$ & Work guarantee of agent $v$ till time $t$, total work guarantee of agent $v$\\
$\pw$ & Rate of pay for every hour worked\\
$\hw_v$ & Total handout of agent $v$ \\
$\wp_v$ & Total work payment of agent $v$ \\
$\cp$ & Platform cost \\
$\pmin$ & Hourly minimum wage guarantee\\
$\war$ & Ratio of total work time to total active time\\
$V^t$ & Set of active agents at window $t$\\
$O^t$ & Set of orders waiting for agent assignment at window $t$\\
$B^t$ & Set of batched orders waiting for agent assignment at window $t$\\
$\ew(v,b)$ & Weight of edge between agent $v$ and order batch $b$ in bipartite matching graph\\
$EDT(o,v)$ & The expected delivery time of order $o$ if it is assigned to agent $v$\\
$SDT(o)$ & The shortest delivery time of order o\\
$XDT(o,v)$ & The excess delivery time of order $o$ if it is assigned to agent $v$\\
\bottomrule
\end{tabular}}
% \vspace{-0.10in}
\caption{Summary of frequently used notations.} 
\label{tab:notations}
% \vspace{-0.20in}
\end{table}
\begin{prob}[Minimize Platform Cost]
\label{prb:cost}
Devise an order to vehicle allocation algorithm such that $\cp$ is minimized. Note that due to the handout component in $\cp$ (Eq.~\ref{eq:cp}), all agents are guaranteed to get their promised income. 
\end{prob}
In this work, %The proposed problem is therefore a \textit{multi-objective optimization} problem. 
we propose to not only %want to 
minimize food delivery time (Prob~\ref{prb:time}), but also the cost to platform (Prob~\ref{prb:cost}).

\begin{prob}[Work Guarantee Problem]
\label{prb:work4food}
Find order to vehicle allocation such that both platform cost $\cp$ and total excess delivery time are minimized while ensuring guaranteed income to all agents.
\end{prob}
The proposed problem is therefore a \textit{multi-objective optimization} problem. 
%This multi-objective optimization 
Moreover, Prob~\ref{prb:work4food} is NP-hard since Prob~\ref{prb:time} is NP-hard~\cite{foodmatch}. Hence, we explore heuristics in the form of bipartite matching.
  %The following section proposes our algorithm \wgalgo, which achieves this goal. %We later compare our results with \fmalgo and \ffalgo to analyze the advantages of using this new strategy.

% \vspace{-0.05in}
\section{\wgalgo: Proposed Algorithm}
\label{sec:algo}

\begin{comment}
\begin{figure}
\centering
% \vspace{-0.10in}
\includegraphics[width=3.5in]{images/flow_chart.pdf}
% \vspace{-0.20in}
\caption{Algorithm Flowchart of \wgalgo.} %only 2 days
\label{fig:flow_chart}
% \vspace{-0.20in}
\end{figure}
\end{comment}

\noindent
 \wgalgo aggregates all incoming orders in a window of size $\Delta = 3$ minutes and any unassigned orders from past windows. As in \fmalgo, we batch orders together if they can be picked and delivered efficiently by the same agent. We use the same batching algorithm as in \fmalgo. We then assign each batch to an active agent. 
 The pseudocode of \wgalgo is provided in Alg.\ref{alg:wgalgo}.
 \begin{algorithm}[t]
    \caption{\wgalgo}
    \label{alg:wgalgo}
    {\scriptsize
    \textbf{Input}: Order stream $O$, Agents and their active intervals $V$\\
    \textbf{Parameter}: Work guarantee $\wgr$
    \begin{algorithmic}[1]
    \STATE unassignedOrders $\leftarrow \emptyset$
    \FOR{each window $t$}
    \STATE $O^t \leftarrow$ GetNewOrders$(t)$ $\cup$ unassignedOrders.
    \STATE $V^t \leftarrow$ GetActiveAgents$(t)$.
    \IF{agent rejection or dynamic guarantee is on} 
      \FORALL{$v \in V^t - V^{t-1}$}
        \STATE $v$.predictedPay $\leftarrow$ GPRegression$(v,t)$
      \ENDFOR
    \ENDIF
    \IF{agent rejection is on} 
      \FORALL{$v \in V^t \setminus V^{t-1}$}
        \IF{$v$.predictedPay $< \wgr\at_v$}
          \STATE Reject agent $v$ since payment guarantee cannot be met.
        \ENDIF
      \ENDFOR
    \ENDIF
    \FORALL{$v \in V^t$}
      \IF{dynamic guarantee is on}
        \STATE $\wgr_v \leftarrow $v$.predictedPay/\at_v$
      \ELSE
        \STATE $\wgr_v \leftarrow \wgr$
      \ENDIF
    \ENDFOR
    \STATE $B^t \leftarrow$ BatchOrders($O^t$).
    %\STATE edges $\leftarrow$ CheckConstraints$(B^t, V^t)$.
    %\STATE $\forall (v,b) \in edges$, $e(v,b) \leftarrow$ GetEdgeWeight$(v,b)$.
    \STATE $\mathcal{B}\gets$Complete bipartite graph $(B^t,V^t)$. Edge weights defined as in Eq.~\ref{eq:ew}.
    \STATE assignments $\leftarrow$ Hungarian$(\mathcal{B})$.
    \STATE unassignedOrders $\leftarrow O^t - $assignments.orders
    \ENDFOR
    \end{algorithmic}}
 \end{algorithm}
%  The pseudocode of \wgalgo is provided in Alg.\ref{alg:wgalgo} in the supplementary.% Figure \ref{fig:flow_chart} highlights the new additions in \wgalgo compared to \fmalgo.

% \vspace{-0.05in}
\subsection{Order Matching}
To assign orders to agents, we perform matching on a \textit{weighted bipartite graph} between order batches $B^t$ and active vehicles $V^t$ in the current window using the \textit{Hungarian} algorithm~\cite{kuhn1955hungarian,munkres1957algorithms}. 

\noindent
{\bf Edge weights in matching graph:} The weight of the edges in the bipartite matching graph at time $t$ between order batches $b \in B^t$ and vehicles $v \in V^t$ is given by
% \vspace{-0.05in}
\begin{equation}
\label{eq:ew}
    \ew(v,b)  =  
    \begin{cases} 
        max\left\{\wt_v^t + \dw_v^b-\wg_v^t,0\right\} & \wg_v^t > \wt_v^t \\
        \dw_v^b & \wg_v^t \leq \wt_v^t
    \end{cases}
\end{equation}
$\dw_v^b$ denotes the extra work agent $v$ will do to deliver batch $b$. The edge weight captures the additional cost to deliver batch $b$ using agent $v$ given the work guarantee. Any work below the work guarantee comes at no extra cost to the platform, as it needs to pay unmet guarantees through handouts. Thus, the edge weights equal the extra work beyond the guarantee. 

The above edge weight does not explicitly optimize delivery times, so we add an extra term corresponding to \fmalgo's edge weight~\cite{foodmatch}.

If $O_v$ is the batch (set) of orders already being carried by agent $v$, then the edge weight in \fmalgo is given by
\begin{equation}
\label{eq:ew_fm}
    \ew_{fm}(v,b)  =  \sum_{o \in b} XDT(o,v) + \sum_{o' \in O_v} \Delta_{XDT}(v,o',b)
\end{equation}

where $\Delta_{XDT}(v,o',b)$ is the change in $XDT$ of order $o' \in O_v$ if batch $b$ is assigned to agent $v$.

\begin{theorem}[Equivalence to \fmalgo]
When batching is turned off, i.e., each agent has a capacity of only one order, and the platform provides no work guarantee to the agents, then \fmalgo and \wgalgo compute the same batch-agent assignments.
\end{theorem}
% \textsc{Proof.}
\begin{proof}
When batching is turned off, then each batch has only one order, i.e., $b=\{o\}$. Further, if the platform provides no work guarantee to the agents, Eq.~\ref{eq:ew} reduces to:
% \vspace{-0.05in}
\begin{equation}
\label{eq:ew_equiv}
  \ew(v,b) = \dw_v^b = EDT(o,v)
\end{equation}
Also, if each agent's capacity is one, Eq.~\ref{eq:ew_fm} reduces to
% \vspace{-0.05in}
\begin{alignat}{2}
\label{eq:ew_fm_equiv}
     \ew_{fm}(v,b) &= \sum_{o \in b} XDT(o,v) = XDT(o,v) \\
     &= EDT(o,v) - SDT(o)
\end{alignat}
Since $SDT(o)$ does not depend on the assignment, finding the minimum weight maximal matching with Eq.~\ref{eq:ew_equiv} and Eq.~\ref{eq:ew_fm_equiv} leads to the same assignment.
\end{proof}
% $\hfill\square$
% \vspace{-0.05in}
\subsection{Setting the Work Guarantee Ratio} 
\label{sec:work_guarantee}
A key component in determining the edge weights of the bipartite graph is the work guarantee to be provided to delivery agents (Eq.~\ref{eq:ew}). This guarantee needs to be determined based on the demand-supply dynamics. To determine this guarantee, we initiate our analysis under the equitability assumption that the same guarantee is provided to all agents, i.e., the guarantee is a function of only the demand-supply and not attributes of the agent. We discuss agent-personalized guarantee in the subsequent section.%the guarantee is a function of only the dethe same guarantee is provided to all \wgalgo's work-guarantee ratio parameter $\wgr_v$ determines the work guarantee for a particular agent $v$. 
%Different agents can have the same or different work guarantee ratios. The factors determining $\wgr_v$ for an agent $v$ include the government-mandated minimum wage, available work, and available agents. Factors like agent ratings and efficiency, which are not part of our dataset, could also be used. A particular case of a fixed work guarantee is to promise no work at all. In this case, we can minimize the total work payment using \wgalgo.

%\subsubsection{Fixed work guarantee}
%The fixed work guarantee setting, wherein all agents get the same hourly work guarantee, is used to ensure we meet government-mandated wage guarantees. 
Let us denote the fixed work guarantee ratio by $\wgr$ (Def.~\ref{def:wgr}), pay per hour worked by $p$ and, the hourly minimum wage guarantee by $\pmin$. Then, we want $ \pw \times \wgr \geq \pmin$. For every $\wgr$, the platform can set an appropriate pay per hour worked $p = p(g) = {\pmin}/{\wgr}$ such that the wage guarantee is met with minimum cost.

\begin{theorem}[Optimal value of $\wgr$]
\label{thm:war}
If $\war = \sum_v\wt_v/\sum_v\at_v$ is the ratio of the total work time to the total active time of the system with the current set of orders and agents, then we claim that $\wgr=\war$ minimizes the platform cost $\cp$.
\end{theorem}
% \vspace{-0.05in}
% \textsc{Proof.}
\begin{proof}
This choice of $\wgr$ is explained by analysing the platform cost $\cp$ for different $\wgr$ values. We can rewrite Eq.~\ref{eq:cp} as:
% \vspace{-0.10in}
\begin{equation}
\label{eq:cpg}
\cp(\wgr) = \frac{\pmin}{\wgr} \sum_v \wt_v + 
        \frac{\pmin}{\wgr} \sum_v \max(0,\wgr \at_v - \wt_v)
\end{equation}    

We consider two cases, $\wgr \leq \war$ and $\wgr > \war$ assuming \wgalgo works ideally, i.e., (i) the total work in the system  $\sum_v \wt_v$ does not depend on $\wgr$; (ii) $\wgalgo$ meets guarantees whenever there is sufficient work; (iii) when work is scarce, agents are not assigned more work than the guarantee as agents with unmet guarantees are preferred by Eq.~\ref{eq:ew}.
% \vspace{-0.10in}
\begin{caseof}
    \case{$\wgr \leq \war$}{
    Here, there is enough work available for all the agents to be able to meet the fixed work guarantee since $\sum_v \wgr \at_v \leq \war \sum_v \at_v = \sum_v \wt_v$. So, the second term in Eq.~\ref{eq:cpg} would be equal to zero and we have $\cp(\wgr) \propto 1/\wgr$. Hence, $\cp(\wgr)$ is minimized when $g=\war$.
    }
    \case{$\wgr > \war$}{
    As there is not enough work for all guarantees to be met, agents will not cross their guarantees, and Eq.~\ref{eq:cpg} would reduce to:
    % \vspace{-0.05in}
    \begin{alignat}{2}
    \label{eq:cpg2}
    \cp(\wgr)_{|\wgr>\war}& = \frac{\pmin}{\wgr} \sum_v \wt_v + 
            \frac{\pmin}{\wgr} \sum_v (\wgr \at_v - \wt_v)\\
    \label{eq:cpg3}
     &= \pmin \sum_v \at_v
    \end{alignat} 
    So, if $\wgr > \war$ then, $\cp$ is independent of $\wgr$. However, in this case, the distribution of work assigned to the agents may be unequal even though they end up earning equally (proportional to their active times). This is not desirable. Hence, we set $\wgr = \war$.
    % $\hfill\square$
    }
\end{caseof}
\end{proof}

\begin{figure}[t]
\centering
\includegraphics[width=1\linewidth]{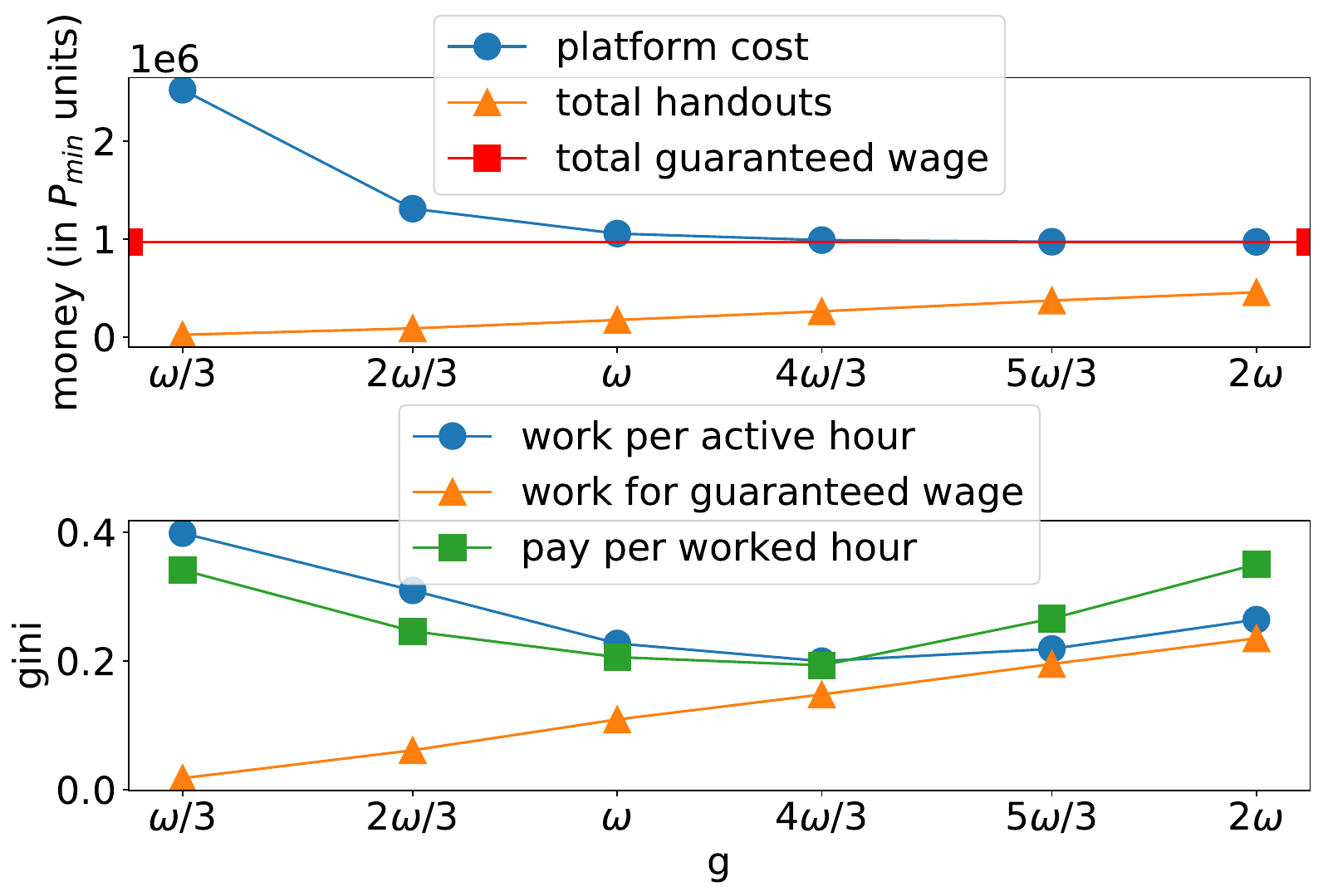}
\caption{Platform cost and Gini score v/s $\wgr$ in city A.} %only 2 days
\label{fig:cost_vs_g}
\end{figure}
% \subsection{Impact of $\wgr$}
% \label{app:guarantee}

Fig.~\ref{fig:cost_vs_g} shows the graph of various measures as a function of work guarantee $\wgr$ for city A. The platform cost behaves as we argued in Theorem \ref{thm:war}, falling initially and then stagnating at $\wgr=\war$. On the other hand, the total handouts keep rising as $\wgr$ increases. We also analyze the fairness with respect to (i) work provided per unit active time, (ii) work done to get the guaranteed wage (some agents may achieve guaranteed wage with the help of handouts), and (iii) pay (including handouts) for every unit of time worked. We can see that these fairness measures get affected if the guarantee is not in sync with the total work in the system, i.e., when $\wgr$ is away from $\war$. This happens because the algorithm does not try to distribute work among agents above the guarantee fairly (when $\wgr < \war$) or among agents below the guarantee (when $\wgr > \war$). However, if the guarantee is close to the available work, agents have little leeway to work more or less than the guarantee, as explained in the two case analyses in Theorem \ref{thm:war}.
% An empirical substantiation of the above theorem is provided in App.~\ref{app:guarantee} in the supplementary.
% \vspace{-0.05in}
\subsection{Estimating $\war_v$ for Each Agent}
From Theorem~\ref{thm:war}, setting $\wgr=\war$ is the optimal choice when all agents get the same guarantee. However, the total work time to total active time varies with time of the day (See Fig.~\ref{fig:war_hour}) since it is a function of several variable factors such as the location of agents, number of agents in the system, and order-density. Hence, it is important to \textit{predict} $\war_v$, i.e., the ratio of expected work to active time when an agent $v$ wants to on-board. Towards that end, we use \textit{Gaussian Process Regression}~\cite{gaussianprocessregression} to model available work $\wt_v$ of agent $v$. As input, the agent provides its active time $\at_v$ on joining the system. So, we have $\war_v = \wt_v/\at_v$ and the work guarantee ratio is set as $\wgr_v = \war_v$ to provide a dynamic guarantee.

%\subsubsection{Dynamic work guarantee}

\begin{figure}
\centering
% \vspace{-0.20in}
\includegraphics[width=0.75\linewidth]{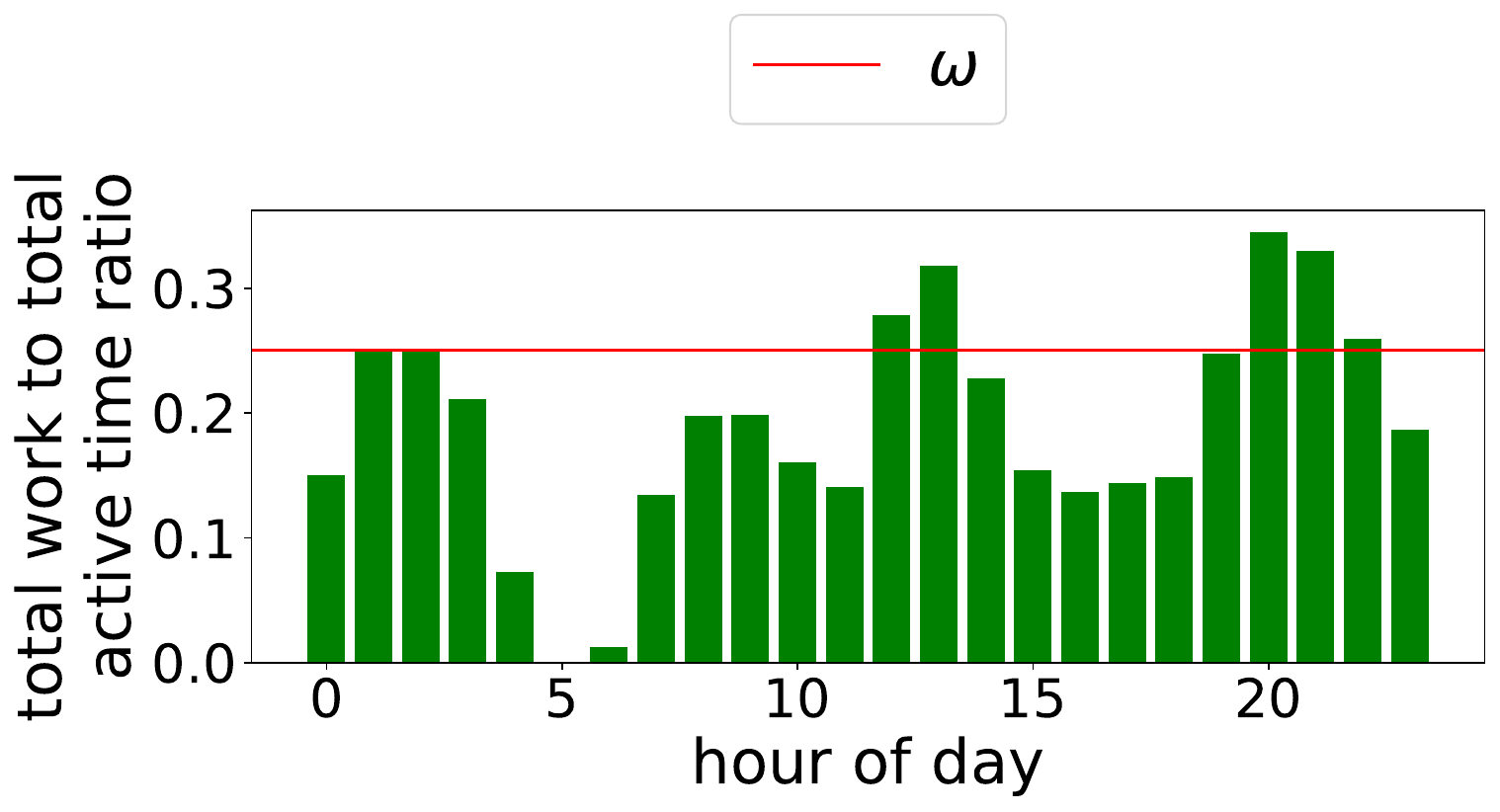}
% \vspace{-0.10in}
\caption{Total work to total active time at different hours in city B.}%of the day in city B.} %only 2 days
\label{fig:war_hour}
% \vspace{-0.20in}
\end{figure}

%Even as the system guarantees a fixed work ratio to the agents, some agents may have more work opportunities at certain times and specific locations. Figure \ref{fig:war_hour} shows the total work to total active time at different hours of the day in city B with \wgalgo. The platform can be transparent and share this with agents beforehand to build trust. A Gaussian process regression model trained on past runs of the fixed work guarantee model proposes the dynamic wage guarantees for each vehicle.

\subsubsection{Gaussian Process Regression}
Gaussian process regression \textsc{(GPR)} is a \textit{Bayesian regression model} used in geo-spatial and time series interpolation tasks. It makes predictions based on the similarity between known points and unknown points. %GPR is well suited for low dimensional data.

For a given train dataset $(\mathbf{X},Y)$ and test points $\mathbf{X_*}$, \textsc{Gpr} predicts the underlying function $f$. In the basic \textsc{Gpr} model, we have:
\begin{alignat}{2}
\nonumber
\begin{bmatrix}
f(\mathbf{X_*})\\
f(\mathbf{X})
\end{bmatrix}
&\sim
\mathcal{N}\left(0,
\begin{bmatrix}
k(\mathbf{X_*}, \mathbf{X_*}) ~~ k(\mathbf{X_*},\mathbf{X})\\
k(\mathbf{X_*},\mathbf{X})^T ~~ k(\mathbf{X},\mathbf{X})
\end{bmatrix}\right)
& \text{(prior)}\\
\nonumber
y(\mathbf{X}) &\sim \mathcal{N}\left(f(\mathbf{X}), \eta^2I\right) &  \text{(likelihood)}
\end{alignat}

\begin{figure*}[t!]
\centering
% \vspace{-0.25in}
\subfloat[][City A]
{\includegraphics[height=1.54in]{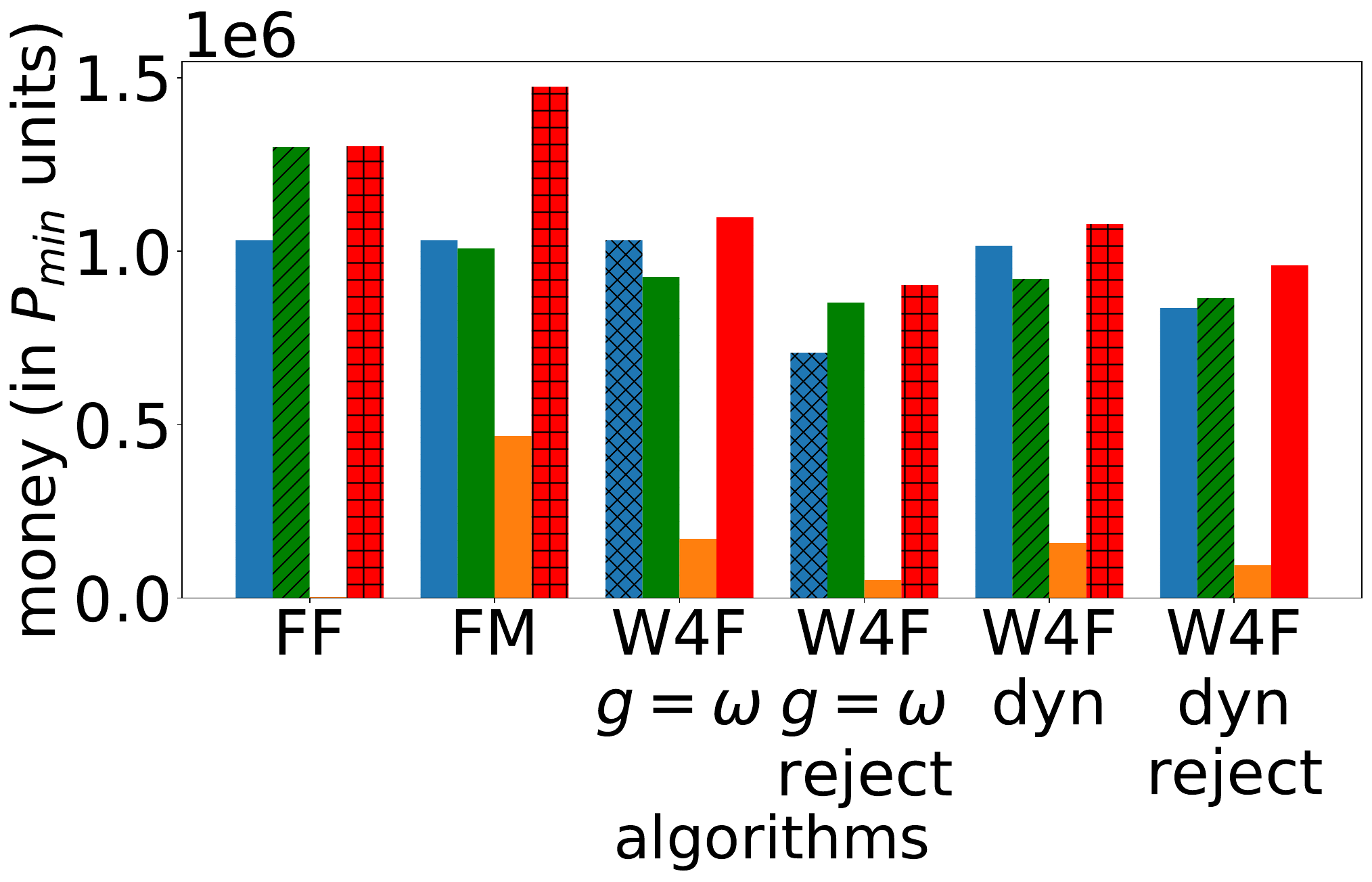}}
\subfloat[][City B]
{\includegraphics[height=1.54in]{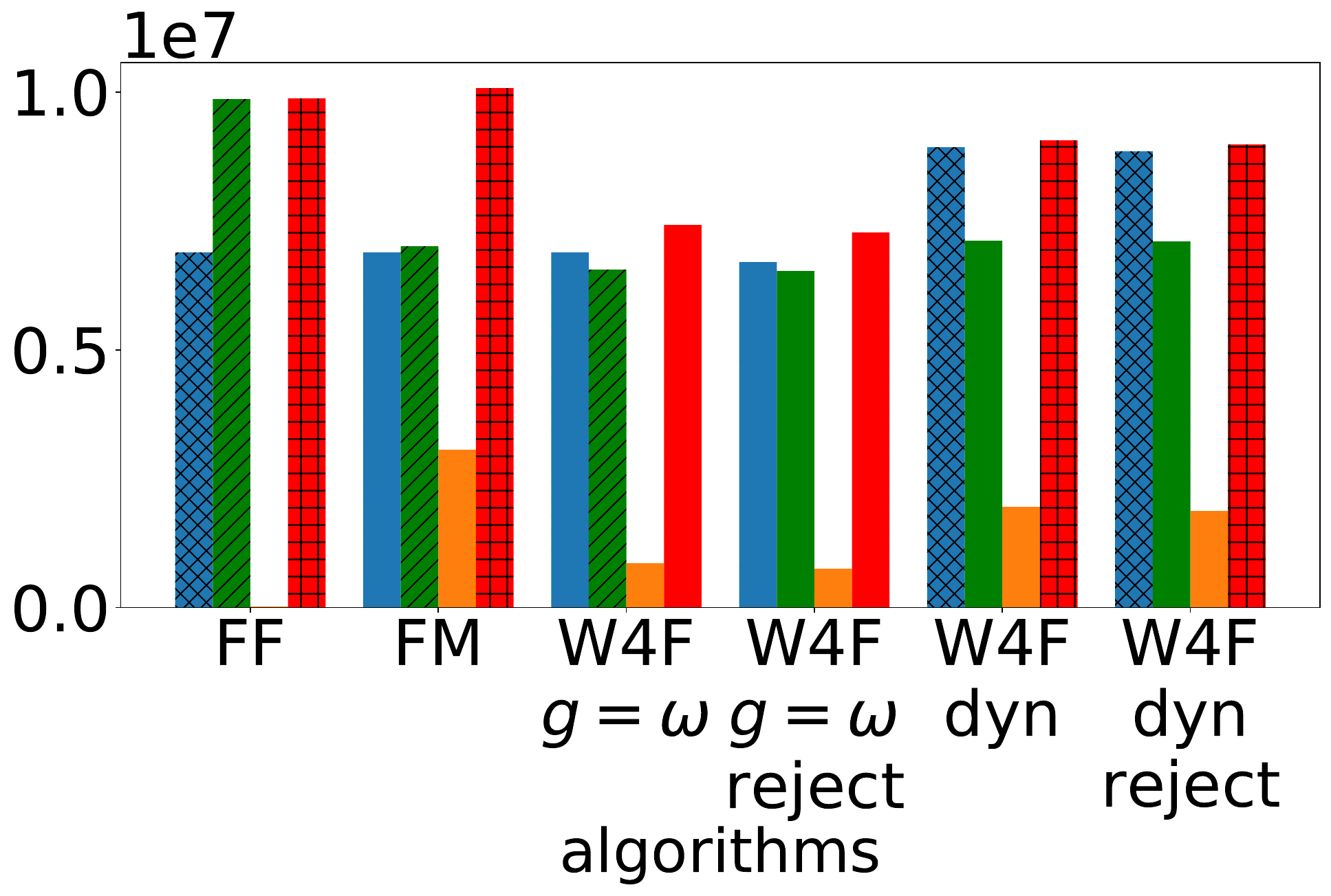}}
\subfloat[][City C]
{\includegraphics[height=1.54in]{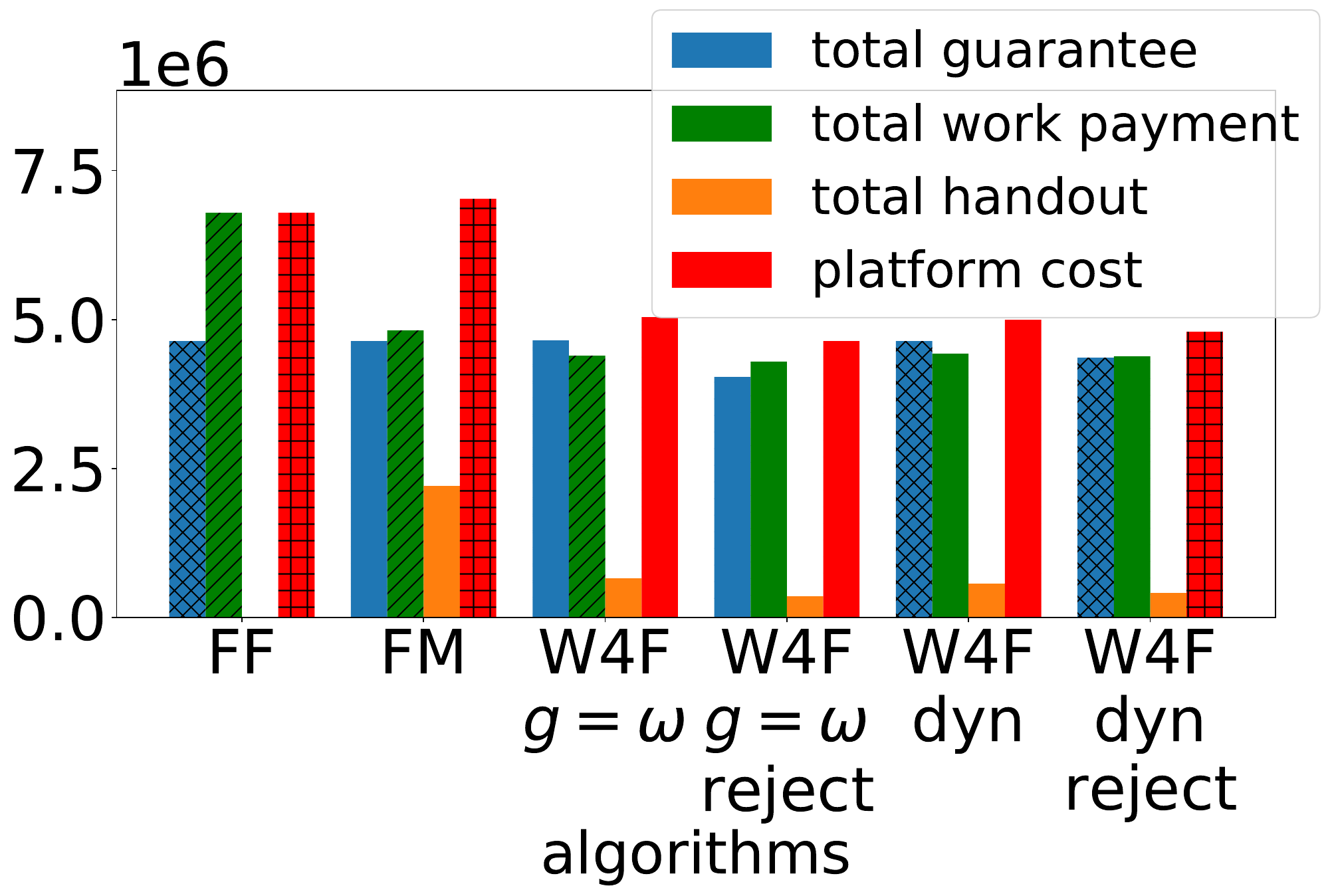}}
% \vspace{-0.10in}
\caption{Cost analysis of different algorithms.} 
\label{fig:cost_analysis_diff_algos}
% \vspace{-0.20in}
\end{figure*}

$y(\mathbf{X})$ and $f(\mathbf{X})$ are random variable vectors. Training data $Y$ is modeled as samples of $y(\mathbf{X})$. $k(x,x')=\sigma^2 \cdot exp(-(x-x')^2/2l^2)$ is the \textit{kernel function} and $\eta^2I$ is the noise co-variance. $\sigma$, $l$ and $\eta$ are model parameters \textit{learned} using \textit{gradient descent} to best explain the training data $(\mathbf{X},Y)$. The predictions of the \textsc{Gpr} model are made by computing $p(f(\mathbf{X_*})\mid y)$ by marginalizing $f(\mathbf{X})$ from $p(f(\mathbf{X_*}), f(\mathbf{X})\mid y)$. We use the \textit{Sparse Variational} version of the \textsc{Gpr} (with \textit{Cholesky Variational Distribution}) for better running times on the large dataset. Variational inference uses a new distribution $q(f(\mathbf{\mathbf{X_*}}), f(\mathbf{X}))$, called the \textit{variational distribution}, to approximate $p(f(\mathbf{X_*}), f(\mathbf{X})\mid y))$. In Sparse \textsc{Gpr}, inducing locations $X_s$ are found to summarize the training data.

In our setting, the \textsc{Gpr} predicts the total work $W_v=\mathbf{X_*}$ an agent $v$ will get during its active time, which is used for getting $\war_v$. We characterize each agent $v\in V^t$ with the feature set $\mathbf{X}=\{\mathbf{X}^v\mid v\in V^t\}$, where $\mathbf{X}^v$ is a vector containing log-in/log-off time, log-in location coordinates in terms of latitude and longitude, number of agents currently in the system, and the number of orders per window. Other features may also be used based on the factors influencing allocation (ex. driver rating).

We note that
%since order assignment may not be highly predictable, using a model which 'interpolates' to give an estimate may be a better idea. Also, note that
we only need an estimate of the available work to provide a guarantee that can be met. Getting the prediction exactly correct is not required. Hence, \textsc{Gpr} is an appropriate choice. In addition, \textsc{Gpr} does not require us to set the functional form of the prediction model or to tune any hyper-parameters. %Finally, we note that if the allocation algorithm considers agent-specific features such as $v$'s ratings, age, etc. to customize edge weights in the bipartite matching component, those features may also be fed to \textsc{Gpr} to generate more accurate predictions.

\noindent
\textbf{Balancing agent-order dynamics:} \fmalgo~\cite{foodmatch} shows that platforms tend to engage more agents than required particularly, on slots that do not correspond to lunch and dinner hours. Over-provisioning drivers results in lower guarantees and consequently may violate the minimum wage guarantee as mandated by law. The prediction from \textsc{Gpr} enables platforms to on-board the optimal number of drivers. Specifically, the system would on-board drivers as long as the expected pay per hour, i.e., $\war_v\times p$, is higher than the minimum pay per hour. 

%If there are more agents than needed at certain times of the day, it is wasteful for the platform to provide minimum work guarantees. So, the platform may choose not to onboard some agents to ensure that all agents in the system can meet their work guarantees without needing handouts. Any agent whose predicted total work when trying to join the system is less than their minimum work requirement is thus rejected. The Gaussian process regression model is used to predict the total work of an agent.

% \vspace{-0.10in}
\section{Empirical Evaluation}
\label{sec:experiments}
% \vspace{-0.05in}
In this section, we benchmark \wgalgo against the baselines of \fmalgo~\cite{foodmatch} and \ffalgo~\cite{fairfoody} and establish that:
\begin{itemize}
    \item \textbf{Practicality:} Compared to \fmalgo and \ffalgo, \wgalgo provides a more practical balance between delivery times and cost-efficiency of the platform.
    \item \textbf{Fairness and sustainability:} \wgalgo generates fair allocations without the pitfalls of strenuous work distribution among agents or additional greenhouse emissions.
\end{itemize}
Our codebase is available at \url{https://github.com/idea-iitd/Work4Food}. Algorithm and simulation implementations are in C\texttt{++} and the Gaussian process regression training is in Python. Our experiments are performed on a machine with Intel(R) Xeon(R) CPU @ 2.10GHz with 252GB RAM on Ubuntu 18.04.3 LTS. 
% Details on the experimental setup is provided at App.~\ref{app:setup} in the Supplementary. 
The datasets are described in \S~\ref{sec:dataset}. We take two days (one weekday and one weekend day) of data to train the \textsc{Gpr} model in each city. We consider both a weekday and a weekend day to learn under different order densities.
\begin{comment}
\subsection{Offline Training of the \textsc{Gpr} Model}
We take two days (one weekday and one weekend day) of data to train the \textsc{Gpr} model in each city. We consider both a weekday and a weekend day to learn under different order densities. To learn under different driver availability, we artificially reduce the number of drivers to 60\%, 70\%, 80\%, and 90\%. We generate the training data by running the fixed guarantee model at $\wgr=\war$. So, we base agent rejections and dynamic guarantees in the other models on the performance of the fixed guarantee model. The trained model is then saved and used by \wgalgo on the remaining four days. We use GPyTorch \cite{gardner2018gpytorch} to implement our \textsc{Gpr} model.
\end{comment}
% \subsubsection{Online Pipeline}

% \begin{algorithm}[tb]
% \caption{\wgalgo - Fixed Guarantee}
% \label{alg:algorithm}
% \textbf{Input}: Order stream $O$, Agents and their active intervals $V$\\
% \textbf{Parameter}: Work guarantee $\wgr$
% \begin{algorithmic}[1]
% \FOR{each window $t$}
% \STATE Let $O^t$ be all orders in t and unassigned past orders.
% \STATE Get all active agents in the window $V^t$.
% \STATE Batch orders $O^t$ to get $B^t$.
% \STATE Check constraints between batches and agents.
% \STATE Calculate edge weights $e(v,b)$.
% \STATE Match batches to vehicles.
% \ENDFOR
% \end{algorithmic}
% \end{algorithm}

% \vspace{-0.05in}
\subsection{Cost to Platform}
%The different algorithms are compared based on several properties like platform cost, delivery times, fairness in work and wage distribution, running times, SLA violations, work times, and CO$_2$ emissions.
 Fig.~\ref{fig:cost_analysis_diff_algos} analyses the monetary cost of each baseline and \wgalgo under several payment guarantees across all three cities. The work guarantees (and therefore payment guarantee) used to evaluate \ffalgo (FF) and \fmalgo (FM) are set to $\wgr_v=\war$ (Recall Theorem~\ref{thm:war}). Hence, the ``total guarantee'' bars for the first three algorithms are the same in Fig.~\ref{fig:cost_analysis_diff_algos}. Recall, any unmet payment guarantees are compensated through handouts. 

The platform cost (red bar) is up to $25\%$ lower with \wgalgo (with $\wgr=\war$) compared to \fmalgo and \ffalgo. This highlights the efficacy of \wgalgo in keeping platform cost low through explicit modeling of pay gap in edge weights (Eq.~\ref{eq:ew}). The cost is highest in \fmalgo since its allocations are skewed towards a minority of agents resulting in large handouts to the remaining agents.

We next focus on the variations of \wgalgo where agents are rejected if payment guarantees are predicted to not be met (reject) as per \textsc{Gpr} and setting $\war$ dynamically through \textsc{Gpr}. We note that when driver rejection is allowed, the handout component decreases by up to $70\%$, and thereby making the system even more cost-efficient. When $\wgr$ is set dynamically based on demand and supply, the income for agents increases. While this naturally leads to an increase in cost to platform (although remains substantially lower than \fmalgo and \ffalgo), dynamic guarantees provide a more transparent working environment for agents.

\subsection{Delivery Times, Equitability and Environmental Impact}
\begin{figure*}[t]
% \vspace{-0.20in}
\begin{minipage}{7in}
% \begin{minipage}[b]{4.5in}
\centering
\scalebox{0.95}{
\begin{tabular}{llrrrrrr}
\toprule
\multirow{2}{*}{\textbf{City}} & \multirow{2}{*}{\textbf{Property}} & \multirow{2}{*}{\textbf{FF}} & \multirow{2}{*}{\textbf{FM}} & \multicolumn{4}{c}{\textbf{W4F}}\\
 &  &  &  & $\wgr=\war$ & $\wgr=\war$ reject & dynamic & dynamic reject\\
\midrule
\multirow{7}{*}{\textbf{A}} & Avg. Delivery Time & 15.98 mins & 15.65 mins & 16.45 mins & 17.1 mins & 16.48 mins & 16.97 mins\\
& SLA Violations (\%) & 0.12 & 0.02 & 0.15 & 0.34 & 0.13 & 0.28\\
& Gini Income/Active Time & 0.07  & 0.6 & 0.1 & 0.24 & 0.25 & 0.36\\
& Gini Work for Min. Wage & 0  & 0.31 & 0.11 & 0.05 & 0.11 & 0.04\\
& Avg. Work per Agent & 2.36 hrs & 1.83 hrs & 1.7 hrs & 1.94 hrs & 1.68 hrs & 1.98 hrs\\
& CO$_2$ Emission & 68 tonnes & 46 tonnes & 46 tonnes & 41 tonnes & 46 tonnes & 42 tonnes\\
& Running Time per Window & 3.3 s & 0.7 s & 2.6 s & 2.5 s & 2.2 s & 2.1 s\\
& Window Overflows (\%) & 0 & 0 & 0 & 0 & 0 & 0\\
\midrule
\multirow{7}{*}{\textbf{B}} & Avg. Delivery Time & 16.08 mins & 15.91 mins & 16.68 mins & 16.7 mins & 16.51 mins & 16.51 mins\\
& SLA Violations (\%) & 0.19 & 0.15 & 0.26 & 0.27 & 0.27 & 0.27\\
& Gini Income/Active Time & 0.09 & 0.63 & 0.11 & 0.12 & 0.33 & 0.33\\
& Gini Work for Min. Wage & 0.01 & 0.29 & 0.09 & 0.08 & 0.08 & 0.07\\
& Avg. Work per Agent & 3.26 hrs & 2.31 hrs & 2.19 hrs & 2.21 hrs & 2.37 hrs & 2.40 hrs\\
& CO$_2$ Emission & 364 tonnes & 245 tonnes & 239 tonnes & 238 tonnes & 256 tonnes & 255 tonnes\\
& Running Time per Window & 25.5 s & 11.4 s & 19.8 s & 19.5 s & 19.2 s & 19.1 s\\
& Window Overflows (\%) & 0 & 0 & 0 & 0 & 0 & 0\\
\midrule
\multirow{7}{*}{\textbf{C}} & Avg. Delivery Time & 16.82 mins & 16.53 mins & 17.35 mins & 17.47 mins & 17.28 mins & 17.34 mins\\
& SLA Violations (\%) & 0.39 & 0.28 & 0.56 & 0.60 & 0.55 & 0.57\\
& Gini Income/Active Time & 0.1 & 0.65  & 0.1 & 0.17 & 0.25 & 0.31 \\
& Gini Work for Min. Wage & 0  & 0.32 & 0.1 & 0.06 & 0.1 & 0.06\\
& Avg. Work per Agent & 2.90  hrs & 2.06  hrs & 1.90  hrs & 1.92  hrs & 1.92  hrs & 1.96 hrs\\
& CO$_2$ Emission & 293 tonnes & 186 tonnes & 183 tonnes & 178 tonnes & 184 tonnes & 182 tonnes\\
& Running Time per Window & 31.9 s & 8.9 s & 23.1 s & 23.1 s & 22.2 s & 21.6 s \\
& Window Overflows (\%) & 0 & 0 & 0 & 0 & 0 & 0\\
\bottomrule
\end{tabular}}
\captionof{table}{Performance analysis.} 
\label{tab:performace_analysis}
\end{minipage}
% \hspace{0.3in}
\end{figure*}
% \begin{minipage}[b]{2in}
\begin{table}[t]
\centering
\scalebox{0.9}{
\begin{tabular}{lrrr}
\toprule
\multirow{2}{*}{\textbf{Property}}  & \multirow{2}{*}{\textbf{FF}} & \textbf{W4F}\\
  &  & $g=4\war/3$\\
\midrule
Avg. Delivery Time  &  15.98 mins &  16.14 mins   \\
Gini Income/log-in Time & 0.07  & 0.02  \\
Avg. Work per Agent &  2.36 hrs  & 1.89 hrs \\
\bottomrule
\end{tabular}}
% \vspace{-0.05in}
\captionof{table}{Performance of \ffalgo v/s \wgalgo in city A.}
\label{tab:ff_wg}
\end{table}
% \vspace{0.15in}
\begin{figure}[t ]
\centering
\includegraphics[width=0.9\linewidth]{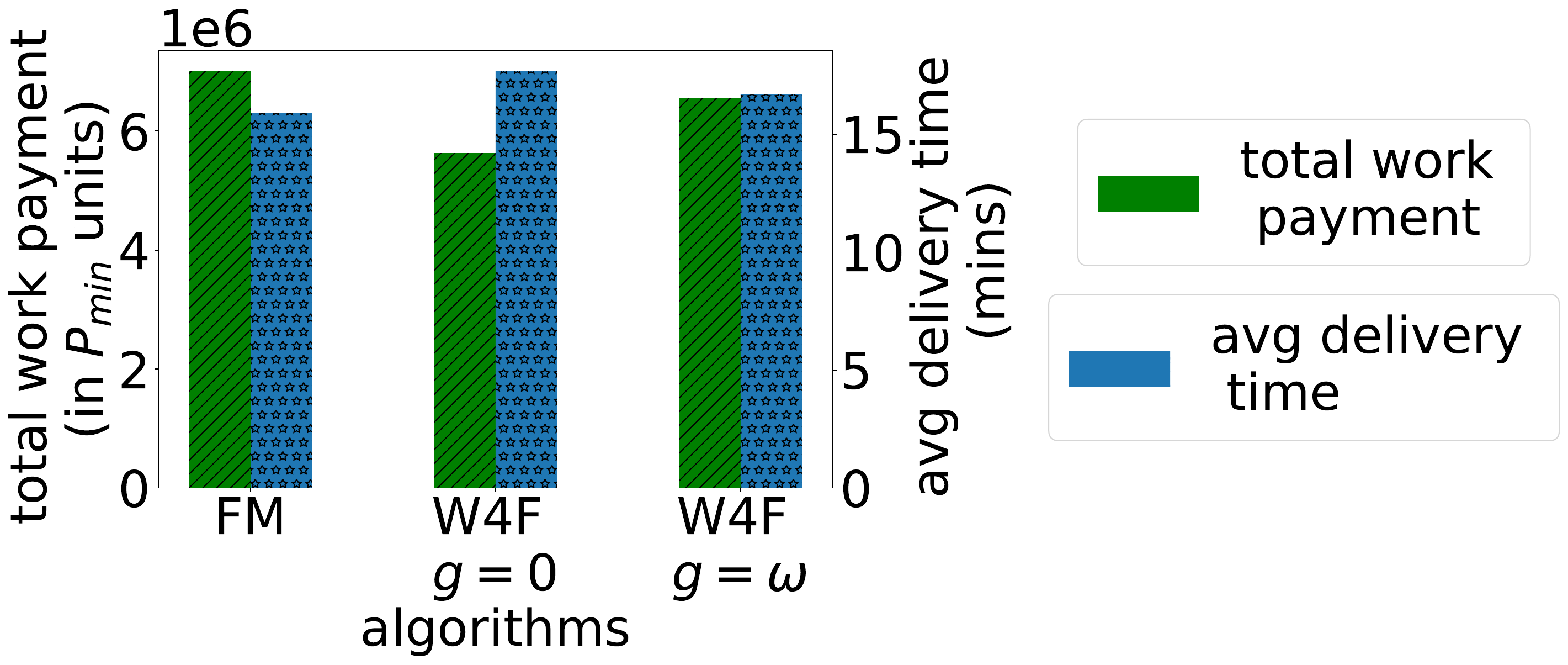}
% \vspace{-0.30in}
\captionof{figure}{Comparing \wgalgo with no guarantee in city B.}
\label{fig:cost_delivery_time_analysis_diff_algos}
\end{figure}
% \end{minipage}
% \vspace{-0.10in}
%\caption{(a) Evaluation of \wgalgo on delivery time, equitability and environmental impact. (b)Running time of \wgalgo.}
% \vspace{-0.10in}
% \end{minipage}
% \end{figure*}

Table \ref{tab:performace_analysis} evaluates \wgalgo and baselines on the above mentioned metrics. The below analysis establishes that \wgalgo provides the best balance between cost-efficiency, delivery time, equitability of work and income, and environmental impact.

\noindent
\textbf{Delivery time:} \fmalgo has the lowest delivery times since it specifically minimizes this factor. However, \wgalgo does not perform poorly. Specifically, all of its variations across all cities are within 1.5 mins of \fmalgo. The agent reject versions of \wgalgo have only slightly higher delivery times, even with a reduced set of agents available for delivery. It points to how the platform keeps under-worked agents in the system even when there is no real impact on customer experience. 
 
 Typically, the food delivery platform provides service level agreements (SLAs), which mandate the time within which each order must be delivered. In our dataset, the SLA was $45$ minutes. We measure the number of SLA violations by each algorithm. As visible in Table~\ref{tab:performace_analysis}, SLA violation is not a concern for any algorithm since it is always below $1\%$.

\noindent
\textbf{Equitability:} Inequality among agents is measured using the Gini score. Lower the Gini, better the equality among agents. We calculate Gini across two component distributions: \textbf{(1)} agents' incomes per unit active time and \textbf{(2)} work done by agents to get to the minimum payment guarantee. 
The results reveal several interesting insights. First, \fmalgo has a highly unequal distribution of work and pay, as seen from the Gini Scores. The Gini income per active time scores of \wgalgo with a fixed guarantee is close to that of \ffalgo, which specifically optimizes the income distribution. However, as well see later, while the excellent Gini of \ffalgo comes at the cost high CO$_2$ emissions and strenuous workload, \wgalgo does not suffer from these issues. %Note again that \fmalgo and \ffalgo do not have any handouts for unmet guarantees. 
Finally, we note that rejecting agents predicted to get low work leads to better equality among agents in terms of the work they do to earn minimum wage. % as handing out free money to under-worked agents would affect this score. 

\noindent
\textbf{Environmental impact:} Even as \ffalgo fairly distributes work, it forces agents to work 30-40\% more on average than \fmalgo, thus making the system highly inefficient. It also leads to higher fuel costs for agents and causes a negative environmental impact. We measure the environmental impact using CO$_2$ emissions calculated from the total distance traveled by agents~\cite{co2_emissions}. \wgalgo, on the other hand, does not lead to higher work times or CO$_2$ emissions to achieve better wage fairness and guarantees. 

\noindent
\textbf{Efficiency:}
In any stream processing algorithm, the number of requests processed per unit of time should be higher than the number of incoming requests in the same period. All three algorithms of \fmalgo, \ffalgo and \wgalgo process the stream in windows of length 3 minutes. We therefore call a window \textit{overflown} if the time taken to process all requests in a window is longer than 3 minutes. In Table~\ref{tab:performace_analysis}, we observe that the number of overflown windows is $0$ across all algorithms in all three cities. Hence, efficiency is not a cause of concern for all three algorithms.
%the running time per windows presents the time taken for each algorithm. The processing and assignment of orders in a window must happen within the window length, or else we say there is a window overflow. All the algorithms we compare have no window overflows. The average running time of windows in \wgalgo is lower than \fmalgo.
% \vspace{-0.05in}
\subsection{Flexibility of \wgalgo}
\begin{comment}
\begin{table}[b]
\centering
\scalebox{0.9}{
\begin{tabular}{lrrr}
\toprule
\multirow{2}{*}{\textbf{Property}}  & \multirow{2}{*}{\textbf{FF}} & \textbf{W4F}\\
  &  & $g=4\war/3$\\
\midrule
Avg. Delivery Time  &  15.98 mins &  16.14 mins   \\
Gini Income/log-in Time & 0.07  & 0.02  \\
Avg. Work per Agent &  2.36 hrs  & 1.89 hrs \\
\bottomrule
\end{tabular}}
\caption{Performance of \ffalgo v/s \wgalgo in city A.} 
\label{tab:ff_wg}
\end{table}
\end{comment}
Fine-tuning the work guarantee parameter $\wgr$ allows customization of \wgalgo towards various needs. In the next discussion, we highlight some of these aspects.

\noindent
\textbf{Fairness:} The work to active time ratio on average in \ffalgo is $4/3$ times that of \wgalgo. We run \wgalgo at $\wgr = 4\war/3$ and compare the results with \ffalgo in Table~\ref{tab:ff_wg} (Recall $\war$ from Theorem~\ref{thm:war}). We observe that \wgalgo at $\wgr = 4\war/3$ not only provides better income equality, but also more relaxed agent workload. These benefits of \wgalgo do not come at the cost of delivery time as it only increases by a minuscule 9 seconds.

\noindent
 \textbf{Platform Cost:} What happens if we provide no work guarantee, i.e., $\wgr=0$?  Fig.~\ref{fig:cost_delivery_time_analysis_diff_algos} answers the question. With no guarantees, \wgalgo is $20\%$ cheaper than \fmalgo but has 1 min 45 s higher average delivery time. This variation of the algorithm can reduce the platform cost and is the cheapest of all algorithms compared. %At this setting, we also note that drivers work for 20\% lesser time on average.

\noindent
\textbf{Ratings-based work guarantee:} One could argue that income equitability is \textit{not} fair. Rather, it should be proportional to the service quality of the agent. We next demonstrate that \wgalgo can easily adapt to this need. 

Unfortunately, our dataset does not contain any rating information. Hence, we simulate ratings by assigning each agent a score uniformly at random from 1 to 5. Based on their ratings, agents are provided a higher or lower work guarantee around the mean $\war$ with the formula: $\wgr_v= (1 + 0.1\times(v.rating-3)) \times \war$. More simply, the highest-rated agents get a $20\%$ higher guarantee than the mean, and the lowest-rated agents get a $20\%$ lower guarantee. We then run \wgalgo with these rating-based agent-specific work guarantees and check whether we are able to satisfy them. 
%Even as we now have a higher control on agent work distribution, there is no adverse effect on any performance or cost measures like platform cost, handouts, and delivery times which remain almost the same as the fixed guarantee model. 
Table~\ref{tab:ratings} presents the results. As we can see, the guarantee and the eventual work per unit time are very closely aligned and thereby validating the efficacy of \wgalgo. 

\begin{table}[t]
\centering
% \vspace{-0.20in}
\scalebox{0.9}{
\begin{tabular}{l|rrrrr}
\toprule
\textbf{Rating}  & 1 & 2 & 3 & 4 & 5 \\
\textbf{$\wgr_v$ of agents} & 0.2 & 0.225 & 0.25 = $\war$ & 0.275 & 0.3\\
\textbf{Average} $\wt_v/\at_v$ & 0.212 & 0.232 & 0.249 & 0.273 & 0.293\\
\bottomrule
\end{tabular}}
% \vspace{-0.10in}
\caption{Performance of \wgalgo with ratings in city B.} 
\label{tab:ratings}
\end{table}

%\input{related}
% \vspace{-0.05in}
\section{Conclusion}
\wgalgo provides a mechanism for food delivery platforms to control and guarantee agent wages while at the same time optimizing their costs and maintaining reasonable delivery times. Optimizing platform cost also leads to a more efficient system with agents not being forced to work more to earn fair wages. It leads to savings in fuel costs and avoids negative environmental impact in order to provide better wages to agents. We also show that by tuning the work guarantee parameter of \wgalgo, it can be customized to address various objectives including minimizing cost, optimizing delivery times, or providing guarantees based on agent performance. Most importantly, \wgalgo injects transparency into the system where agents know the income promised to them and the service provider is assured that cost would be minimized without hampering delivery times.

\clearpage
{
%\scriptsize
\bibliographystyle{named}
\bibliography{main}

\begin{thebibliography}{}

\bibitem[\protect\citeauthoryear{Fairwork}{2021}]{fairwork}
Fairwork.
\newblock Fairwork {India} ratings 2020: Labour standards in the platform
  economy.
\newblock
  \url{https://fair.work/wp-content/uploads/sites/131/2021/01/Fairwork_India_2020_report.pdf},
  2021.

\bibitem[\protect\citeauthoryear{Gupta \bgroup \em et al.\egroup
  }{2022}]{fairfoody}
Anjali Gupta, Rahul Yadav, Ashish Nair, Abhijnan Chakraborty, Sayan Ranu, and
  Amitabha Bagchi.
\newblock Fairfoody: Bringing in fairness in food delivery.
\newblock In {\em Proc.~AAAI}, 2022.

\bibitem[\protect\citeauthoryear{Joshi \bgroup \em et al.\egroup
  }{2021}]{foodmatch}
Manas Joshi, Arshdeep Singh, Sayan Ranu, Amitabha Bagchi, Priyank Karia, and
  Puneet Kala.
\newblock Batching and matching for food delivery in dynamic road networks.
\newblock In {\em Proc.~ICDE}, 2021.

\bibitem[\protect\citeauthoryear{Joshi \bgroup \em et al.\egroup
  }{2022}]{foodmatch_full}
Manas Joshi, Arshdeep Singh, Sayan Ranu, Amitabha Bagchi, Priyank Karia, and
  Puneet Kala.
\newblock Foodmatch: Batching and matching for food delivery in dynamic road
  networks.
\newblock {\em ACM Trans. Spatial Algorithms Syst.}, 8(1), mar 2022.

\bibitem[\protect\citeauthoryear{Khumalo}{2022}]{foodDeliverySA}
Sibongile Khumalo.
\newblock Food delivery apps bring new hope to {SA}'s jobless, but it's not
  always a piece of cake.
\newblock
  https://www.news24.com/fin24/companies/food-delivery-apps-bring-new-hope-to-sas-jobless-but-its-not-always-a-piece-of-cake-20220205,
  2022.

\bibitem[\protect\citeauthoryear{Kottakki \bgroup \em et al.\egroup
  }{2020}]{Kottakki2020CustomerED}
Krishna~Kumar Kottakki, Sunil Rathee, Kranthi~Mitra Adusumilli, Jose Mathew,
  Bharath Nayak, and Saket Ahuja.
\newblock Customer experience driven assignment logic for online food delivery.
\newblock In {\em Proc. IEEE IEEM}, 2020.

\bibitem[\protect\citeauthoryear{Kuhn}{1955}]{kuhn1955hungarian}
Harold~W Kuhn.
\newblock The {Hungarian} method for the assignment problem.
\newblock {\em Naval research logistics quarterly}, 2(1-2):83--97, 1955.

\bibitem[\protect\citeauthoryear{Mitra}{2022}]{swiggyStrike}
Debraj Mitra.
\newblock Food delivery chain partners strike in {Kolkata}, seek higher pay and
  bar on outsourcing.
\newblock
  https://www.telegraphindia.com/my-kolkata/news/food-delivery-chain-partners-strike-in-parts-of-kolkata-seek-higher-pay-and-bar-on-outsourcing/cid/1848505,
  2022.

\bibitem[\protect\citeauthoryear{Moran}{2018}]{co2_emissions}
Greg Moran.
\newblock Emissions 101: A two-wheeler takedown.
\newblock {\em Economic Times}, 2018.

\bibitem[\protect\citeauthoryear{Munkres}{1957}]{munkres1957algorithms}
James Munkres.
\newblock Algorithms for the assignment and transportation problems.
\newblock {\em Journal of the society for industrial and applied mathematics},
  5(1):32--38, 1957.

\bibitem[\protect\citeauthoryear{Newson and Krumm}{2009}]{mapmatch}
Paul Newson and John Krumm.
\newblock Hidden {Markov} map matching through noise and sparseness.
\newblock In {\em Proc. ACM SIGSPATIAL}, 2009.

\bibitem[\protect\citeauthoryear{Reyes \bgroup \em et al.\egroup }{2018}]{mdrp}
Dami{\'a}n Reyes, Alan~L. Erera, Martin W.~P. Savelsbergh, Sagar Sahasrabudhe,
  and Ryan~J. O'Neil.
\newblock The meal delivery routing problem.
\newblock Optimization Online, 2018.

\bibitem[\protect\citeauthoryear{Sodhi}{2021}]{foodDeliveryNewsMinute}
Tanishka Sodhi.
\newblock We are slaves to them: Zomato, {Swiggy} delivery workers speak up
  against unfair practices.
\newblock
  https://www.newslaundry.com/2021/08/14/we-are-slaves-to-them-zomato-swiggy-delivery-workers-speak-up-against-unfair-practices,
  2021.

\bibitem[\protect\citeauthoryear{Ulmer \bgroup \em et al.\egroup
  }{2021}]{Ulmer2021TheRM}
Marlin~W. Ulmer, Barrett~W. Thomas, Ann~Melissa Campbell, and Nicholas Woyak.
\newblock The restaurant meal delivery problem: Dynamic pickup and delivery
  with deadlines and random ready times.
\newblock {\em Transp. Sci.}, 55:75--100, 2021.

\bibitem[\protect\citeauthoryear{Weng and Yu}{2021}]{Weng2021LaborrightPD}
Wentao Weng and Yang Yu.
\newblock Labor-right protecting dispatch of meal delivery platforms.
\newblock In {\em Proc. IEEE CDC}, 2021.

\bibitem[\protect\citeauthoryear{Wilks}{2022}]{turkeyStrike}
Andrew Wilks.
\newblock Turkey's food delivery couriers latest to strike amid economic
  crisis.
\newblock
  https://www.al-monitor.com/originals/2022/02/turkeys-food-delivery-couriers-latest-strike-amid-economic-crisis,
  2022.

\bibitem[\protect\citeauthoryear{Williams and
  Rasmussen}{2006}]{gaussianprocessregression}
C.~K. Williams and C.~E. Rasmussen.
\newblock Gaussian processes for {Machine Learning}.
\newblock 2(3), 2006.

\bibitem[\protect\citeauthoryear{Xue \bgroup \em et al.\egroup
  }{2021}]{Xue2021OptimizationOR}
Guiqin Xue, Zheng Wang, and Guangwei Wang.
\newblock Optimization of rider scheduling for a food delivery service in o2o
  business.
\newblock {\em Journ. of Adv. Trans.}, 2021.

\bibitem[\protect\citeauthoryear{Yildiz and
  Savelsbergh}{2019}]{article:mealDelivery}
Baris Yildiz and Martin Savelsbergh.
\newblock Provably high-quality solutions for the meal delivery routing
  problem.
\newblock {\em Transportation Science}, 53, 07 2019.

\bibitem[\protect\citeauthoryear{Zeng \bgroup \em et al.\egroup
  }{2019}]{10.14778/3368289.3368297}
Yuxiang Zeng, Yongxin Tong, and Lei Chen.
\newblock Last-mile delivery made practical: An efficient route planning
  framework with theoretical guarantees.
\newblock {\em Proc. VLDB Endow.}, 13(3):320–333, 2019.

\bibitem[\protect\citeauthoryear{Zhou and others}{2020}]{zhou2020digital}
Irene Zhou et~al.
\newblock Digital labour platforms and labour protection in {China}.
\newblock Technical report, International Labour Organization, 2020.

\end{thebibliography}
}
% \bibliography{main}
% \input{FairnessGuaranteeIJCAI-ECAI-2022/appendix}
\end{document}